\documentclass[letterpaper, 10pt, conference]{ieeeconf}

\makeatletter

\let\proof\@undefined
\let\endproof\@undefined
\makeatother

\IEEEoverridecommandlockouts  

\usepackage{graphicx}
\usepackage{epstopdf}
\usepackage{lipsum}
\usepackage{amsthm}
\usepackage{amsfonts}
\usepackage{mathtools}  
\usepackage{bm}  
\usepackage[dvipsnames]{xcolor}  

\usepackage{paralist}  
\usepackage{tikz}
\usepackage{caption}  
\usepackage{subcaption}  
\usepackage{cite}
\usepackage{booktabs}  
\usepackage{siunitx}  
\usepackage[bookmarks]{hyperref}  
\hypersetup{
    colorlinks,
    citecolor=black,
    filecolor=black,
    linkcolor=black,
    urlcolor=black,
    pdfauthor={},
    pdfsubject={},
    pdftitle={}
}
\usepackage{xspace}  
\usepackage{citesort}
\usepackage{paralist}

\theoremstyle{definition}
\newtheorem{definition}{Definition}
\newtheorem{theorem}{Theorem}
\newtheorem{assumption}{Assumption}
\usepackage{soul}
\newcommand{\wrt}{w.r.t.\xspace}

\newcommand{\eg}{e.g.\xspace}

\usepackage{acro}
\DeclareAcronym{CBF}{
    short = CBF,
    long = Control Barrier Function
}
\DeclareAcronym{CLF}{
    short = CLF,
    long = Control Lyapunov Function
}
\DeclareAcronym{HJB}{
    short = HJB,
    long = Hamilton-Jacobi-Bellman
}
\DeclareAcronym{QP}{
    short = QP,
    long = Quadratic Programming
}
\DeclareAcronym{RL}{
    short = RL,
    long = Reinforcement Learning
}
\DeclareAcronym{SDRE}{
    short = SDRE,
    long = State Dependent Riccati Equation
}
\DeclareAcronym{MPC}{
    short = MPC,
    long = Model Predictive Control
}

\graphicspath{{imgs/}}
\def\baselinestretch{1}

\title{\LARGE \bf Neural Control Barrier Functions for Safe Navigation}
\author{Marvin Harms, Mihir Kulkarni, Nikhil 
Khedekar, Martin Jacquet, Kostas Alexis
	\thanks{Autonomous Robots Lab, Norwegian University of Science and Technology (NTNU), Trondheim, Norway,
    {\tt \footnotesize
        \href{mailto:marvin.c.harms@ntnu.no}{marvin.c.harms@ntnu.no}}
    }
    \thanks{This work was partially supported by the European Commission Horizon Europe project DIGIFOREST (EC 101070405).}
}

\begin{document}
\maketitle

\begin{abstract}
    Autonomous robot navigation can be particularly demanding, especially when the surrounding environment is not known and safety of the robot is crucial.
    This work relates to the synthesis of Control Barrier Functions (CBFs) through data for safe navigation in unknown environments. A novel methodology to jointly learn CBFs and corresponding safe controllers, in simulation, inspired by the State Dependent Riccati Equation (SDRE) is proposed. The CBF is used to obtain admissible commands from any nominal, possibly unsafe controller. An approach to apply the CBF inside a safety filter without the need for a consistent map or position estimate is developed. Subsequently, the resulting reactive safety filter is deployed on a multirotor platform integrating a LiDAR sensor both in simulation and real-world experiments. 

\end{abstract}
\acresetall  

\section{Introduction}\label{sec:intro}

A host of robotics applications involve the deployment of safety-critical systems in unknown environments. Examples include ground and flying robots deployed for subterranean exploration, search-and-rescue, and forest monitoring~\cite{CerberusScience,valavanis2015handbook}.
Such deployments remain challenging, especially as onboard localization and mapping is prone to latency, noise and drift, possibly resulting in estimation failure and thus potentially collisions.
Responding to this fact, a set of methods have enabled data-driven approaches for collision-free flight without the need of a map. Indicative works relate to privileged learning from an expert policy~\cite{loquercio2021learning}, learned collision prediction with motion primitives~\cite{florence2020integrated,nguyen2023uncertainty} or Model Predictive Control~\cite{jacquet2024n}, and \ac{RL} using depth maps~\cite{kulkarni2024reinforcement}. However, these methods typically do not provide formal guarantees for collision avoidance. 

Another avenue for safe navigation, particularly focused on provable assurances, is the use of safety filters~\cite{wabersich2023data}, which are last resort policies processing controls from any nominal controller in order to render them safe \wrt given criteria such as collision avoidance.
To that end, predictive safety filters~\cite{WABERSICH2021109597} provide a solution for safe navigation, but at a large computational cost.
Alternatively, safety filters based on \acp{CBF} are computationally cheap and have recently received significant attention~\cite{ames2016control,ames2019control}.

\begin{figure}[h]
    \centering
    \includegraphics[width=0.98\columnwidth]{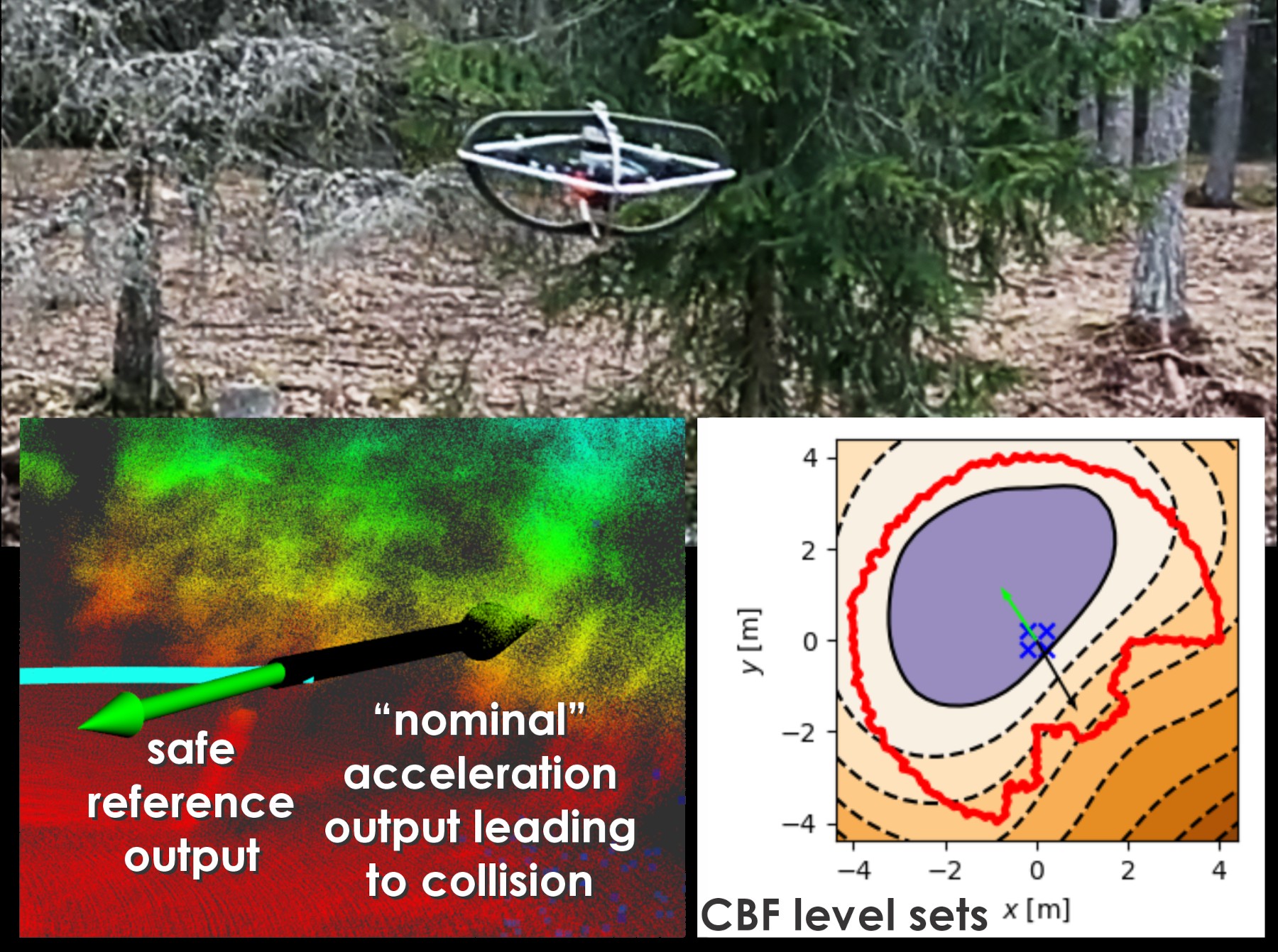} 
    \vspace{-1ex}
    \caption{Instance of the presented experiments for safe navigation through direct learning of neural control barrier functions commanding accelerations. Video: \url{https://www.youtube.com/watch?v=1id0C6jiFEg}}
    \label{fig:intro}
    \vspace{-4ex}
\end{figure}

Nevertheless, despite the recent advances~\cite{ferraguti2022safety,unlu2023control}, finding a permissive \ac{CBF} for constrained high-order systems remains an open challenge.
Machine learning has thus also been leveraged to learn \acp{CBF} for such systems from data~\cite{dawson2023safe,wabersich2023data}.
Such techniques allow to formalize a \ac{CBF} as safety \wrt collision avoidance in unknown environments~\cite{dawson2022learning,keyumarsi2023lidar}.
However, learning \acp{CBF} requires satisfying forward invariance conditions, which is particularly challenging.
Previous works either lifted this requirement by considering oversimplistic first-order systems~\cite{dawson2022learning}, or employed analytic \acp{CBF} which in turn struggle to scale to cluttered environments~\cite{keyumarsi2023lidar}.
Some approaches also necessarily employ some prior control policy, along which the \ac{CBF} is learned \cite{dawson2022safe}.

Motivated to overcome these limitations, in this work we propose a novel method for learning \acp{CBF} for safe navigation of aerial robots integrating range sensing. The contributions are threefold:
First, we introduce a novel formulation to jointly learn \acp{CBF} and safe controllers for collision avoidance in unknown environments for high-order systems with input constraints, using an adaptation of the \ac{SDRE}.
The proposed approach is trained by sampling the observation space of range sensors and the vehicle state space entirely in simulation.
Second, we propose a policy to reactively apply the \ac{CBF} in a safety filter without requiring approximation of the often unknown observation dynamics.
This policy operates in a local frame, thus relaxing the need for a consistent map.
Finally, the proposed safety filter is both statistically evaluated in simulations and experimentally verified on an autonomous quadrotor flying in novel environments (\eg, Fig.~\ref{fig:intro}).
The results demonstrate that the learned neural \acp{CBF} satisfy safety requirements, in terms of formulation guarantees and practical results, and enable practically safe autonomous navigation exploiting only instantaneous range observations and velocity estimates. 

In the remaining paper, Section~\ref{sec:related} presents related work, and~\ref{sec:preliminaries} preliminaries. The proposed neural \acp{CBF} are detailed in Section~\ref{sec:ncbf} and their use for navigation in Section~\ref{sec:quadrotor}. Evaluations are in Section~\ref{sec:xp} and conclusions in Section~\ref{sec:concl}.

\section{Related Work}\label{sec:related}
This contribution relates to the body of work in \acp{CBF}, their application to safe robot navigation, and \acp{CBF} learned from data.
\acp{CBF} were introduced to ensure safety for critical systems~\cite{ames2016control,ames2019control} and are becoming increasingly prevalent in robotics~\cite{ferraguti2022safety}.
Focusing on navigation, \cite{singletary2021comparative} compares artificial potential fields and \acp{CBF}. Using DBSCAN on LiDAR maps to cluster obstacles and thus derive state constraints, \cite{jian2023dynamic} utilizes \acp{CBF} to demonstrate safety for a ground robot. Considering known obstacles locations,~\cite{jang2024safe} develops safe control with multiple Lyapunov-based \acp{CBF}. The method in~\cite{unlu2023control} assumes a volumetric map combined with \acp{CBF} for safe aerial robot navigation demonstrated in simulation,
while~\cite{zhou2024control} enables safe teloperation experimentally. 


Focusing on the benefits of data-driven methods, learning-based approaches exploiting \ac{CBF} certificates have expanded over the recent years~\cite{dawson2023safe,wabersich2023data}.
Machine learning can enable the use of \acp{CBF} for complex systems for which deriving analytic safety certificates is challenging.
A first approach to mention is Safe \ac{RL},
where a \ac{CBF} is jointly learned during training to guide the policy to implicitly encode the underlying safety mechanism~\cite{cheng2019end}.
In~\cite{choi2020reinforcement}, \ac{RL} is proposed to learn the residuals in the dynamics of a \ac{CBF} (and a \ac{CLF}) with a neural network.
Aside \ac{RL}, supervised learning is also employed to derive safety filters based on \ac{QP}~\cite{taylor2020learning,dawson2022safe}.

Applying learning-based \acp{CBF} to safe navigation, the \ac{CBF} is used to encode safety \wrt the environment. The main challenge is to synthesize a \ac{CBF} based on partial observations of the environment obtained through exteroceptive sensors.
An approach is to extract features from the observations (\eg, the distance to the closest obstacle~\cite{song2022safe,keyumarsi2023lidar}) or geometric information about the obstacles in state space~\cite{xiao2023barriernet}.
However, these approaches struggle to scale to high-dimensional inputs (\eg, images, LiDAR scans) and therefore tend to not generalize well to unknown environments.
Leveraging deep learning to lift this limitation, \cite{abdi2023safe} uses a Generative Adversarial Network to detect the 3D positions of obstacles in images and infer their time derivatives used to compute a geometric \ac{CBF}.
In~\cite{tong2023enforcing} a NeRF predicts future images which will result from a given action, allowing to approximate the gradient of the \ac{CBF}. Yet, this is computationally demanding since the NeRF is queried for each sampled action.
Alternatively, the authors in~\cite{dawson2022learning} learn a \ac{CBF} within the observation space with the \ac{CBF} approximated by a neural network trained by sampling states and observations. However, this requires approximating the temporal gradients of observations which is particularly challenging in non-trivial environments.
Finally, \cite{srinivasan2020synthesis,long2021learning} learn a Signed Distance Function then used as a \ac{CBF}.

\section{Preliminaries}\label{sec:preliminaries}

This section outlines relevant background knowledge. 
\subsection{Symbols and Abbreviations}

\small
\begin{tabbing}
 \hspace*{2.2cm} \= \kill
  $\mathbf{x} \in \mathcal{X}$,~$\mathbf{u} \in \mathcal{U}$ \>  state $(\mathbf{x})$ and input $(\mathbf{u})$ vector \\[0.5ex] 
  $\mathbf{o} \in \mathcal{O}$ \>  observation vector \\[0.5ex] 
  $\mathcal{L}_f h, \mathcal{L}_B h \mathbf{u}$ \>  Lie derivatives of $h$ along $f$, $B \mathbf{u}$ \\[0.5ex] 
  $\frac{d}{dt} V$, $\nabla_\mathbf{x} V$ \>  time derivative of $V$, gradient of $V$ w.r.t. $\mathbf{x}$ \\[0.5ex] 
  $\mathbf{a} \cdot \mathbf{b}$ \>  dot product of $\mathbf{a}$ and $\mathbf{b}$ \\[0.5ex] 
  $\mathbb{I} $ \>  identity matrix \\[0.5ex] 
  $\|\cdot\|$ \>  euclidean norm \\[0.5ex] 
  $\text{eig}(\cdot)$ \>  eigenvalues \\[0.5ex] 
  $\text{proj}_{\mathcal{U}}(\cdot)$ \> projection onto $\mathcal{U}$ \\[0.5ex]
  $\text{ReLU}, \text{ELU}$ \>  rectified linear unit, exponential linear unit  \\[0.5ex] 
\end{tabbing}
\normalsize

\subsection{System Definition}
We consider the control affine system:

\small
\begin{equation}\label{system}
    \frac{d}{dt} \mathbf{x} = \underbrace{A(\mathbf{x})\mathbf{x}}_{f(\mathbf{x})} + B(\mathbf{x}) \mathbf{u}
\end{equation}
\normalsize
where $\mathbf{x} \in \mathcal{X} \subseteq \mathbb{R}^n$ and $\mathbf{u} \in \mathcal{U} \subseteq \mathbb{R}^m$. Further, we divide the state space $\mathcal{X}$ into the compact, open subset $\mathcal{X}_\text{free}$ and the closed subset $\mathcal{X}_\text{obs}$ such that $\mathcal{X} = \mathcal{X}_\text{free} \cup \mathcal{X}_\text{obs}$, where $\mathcal{X}_\text{obs}$ denotes the set of states we seek to avoid (the constraint set). Further we assume that the origin is safe, i.e. $\mathbf{x}=0 \in \mathcal{X}_\text{safe}$, where $\mathcal{X}_\text{safe} \subseteq \mathcal{X}_\text{free}$ is a \textit{control invariant} set.

\begin{definition}[Control Invariant Set]
    A set $\mathcal{X}_\text{safe}$ is a forward control invariant set for the system \eqref{system}, if for any $\mathbf{x}({t_0}) \in \mathcal{X}_\text{safe}$ there exists an input trajectory $\mathbf{u}(t) \in \mathcal{U}$ such that $\mathbf{x}({t}) \in \mathcal{X}_\text{safe} \quad \forall t\geq t_0$ under the system \eqref{system}.
\end{definition}
The use of control invariant sets is a central concept in safe control as it allows to bound areas in $\mathcal{X}$ where future constraint satisfaction (avoiding $\mathcal{X}_\text{obs}$) can be guaranteed.

\subsection{Control Barrier Functions}
Following the definition of a safe set, we now define corresponding control barrier functions which can later be used to synthesize safe control actions. 
\begin{definition}[Control Barrier Function \cite{ames2019control}]\label{def_CBF}
    Let $\mathcal{X}_\text{safe}$ be the $0$-superlevel set of a continuously differentiable function
    $h: \mathbb{R}^n \rightarrow \mathbb{R}$ with the property that $\mathcal{X}_\text{safe} = \{ \mathbf{x} \in \mathcal{X} | h(\mathbf{x})\geq 0 \}$ and $\{\mathbf{x} \in \mathcal{X} | \frac{dh}{d\mathbf{x}}(\mathbf{x}) = 0\} \cap \{\mathbf{x} \in \mathcal{X} | h(\mathbf{x}) = 0\} = \emptyset$. Then, $h$ is a control barrier function if there exists an extended class $\mathcal{K}_\infty$ function $\alpha(\cdot)$ such that for the system \eqref{system}:
    
    \small
    \begin{subequations}
        \begin{equation}\label{CBF_obs_condition}
            h(\mathbf{x})<0 \quad \forall \mathbf{x} \in \mathcal{X}_\text{obs}
        \end{equation}
        \begin{equation}\label{CBF_lie_condition}
            \underset{\mathbf{u} \in \mathcal{U}}{\text{sup}} [\mathcal{L}_f h(\mathbf{x}) + \mathcal{L}_B h(\mathbf{x}) \mathbf{u}] \geq -\alpha(h(\mathbf{x}))
        \end{equation}
    \end{subequations}
    \normalsize
    for all $\mathbf{x} \in \mathcal{X}_\text{safe}$. Further, an input $\mathbf{u}$ is considered safe with respect to a valid \ac{CBF}, if it satisfies \eqref{CBF_lie_condition}.
\end{definition}

By definition, a \ac{CBF} creates a safety certificate such that by control invariance, all state trajectories starting within $\mathcal{X}_\text{safe}$ can be kept inside $\mathcal{X}_\text{safe}$ for all future times by choosing appropriate input trajectories. Asymptotic stability of $\mathcal{X}_\text{safe}$ can  be achieved by extending requirement \eqref{CBF_lie_condition} to hold on $\mathcal{X}$.
Further, \acp{CBF} also provide the informative condition \eqref{CBF_lie_condition} on a candidate input at any given state, which can be utilized to synthesize safe control actions form unsafe ones. In \cite{ames2016control}, a simple yet effective way to employ condition \eqref{CBF_lie_condition} in a \ac{QP} to create a computationally cheap safety filter is proposed, yielding the safe input $\mathbf{u}^*$ from an arbitrary (possibly unsafe) reference input $\mathbf{u}_\text{ref}$:

\small
\begin{equation}\label{safety_QP}
    \begin{aligned}
    \mathbf{u}^*  = \quad & \underset{\mathbf{u} \in \mathcal{U}}{\text{argmin} } \quad ||\mathbf{u} - \mathbf{u}_\text{ref}||^2\\
        & \textrm{s.t.}  \quad \mathcal{L}_f h(\mathbf{x}) + \mathcal{L}_B h(\mathbf{x}) \mathbf{u} \geq -\alpha(h(\mathbf{x}))\\
    \end{aligned}
\end{equation}
\normalsize
While not being optimal in terms of closed-loop performance, the reactive policy \eqref{safety_QP} allows to synthesize safe control actions from unsafe ones. 
Therefore \acp{CBF} offer an efficient alternative to predictive safety filters \cite{WABERSICH2021109597}, especially for high-order systems with complex constraints, for which the computational complexity can become large.

\section{Neural Navigation CBFs}\label{sec:ncbf}

The above motivates the use of \acp{CBF} to navigation systems for autonomous robots. Aiming for a safety filter applied on a multirotor aerial robot with odometry state $\mathbf{x}$, action vector $\mathbf{u}$ subject to constraints $||\mathbf{u}|| \leq u_\text{max}$, and a local instantaneous observation of the environment $\mathbf{o}$ (e.g., a scan from an onboard LiDAR), we proceed with a generic formulation of the proposed neural navigation 
\acp{CBF} before specializing to the targeted application in Section~\ref{sec:quadrotor}. To that end, dynamics are considered to obey the form in~\eqref{system} as applicable for the case of translational dynamics of attitude-control multirotors. Accordingly, our goal is to enable collision-free navigation of the system \eqref{system} in arbitrary, unknown, static environments. Subsequently, we develop an approach to train and apply \acp{CBF} for safe navigation.

\subsection{Constructive Safety from Observation CBFs}
In contrast to the traditional formulation of \acp{CBF}, where $h$ is a function only of $\mathbf{x}$, we define the CBF as a function of both the exteroceptive observations $\mathbf{o} \in \mathcal{O} \subset \mathbb{R}^o$ and the state $\mathbf{x}$ as $h(\mathbf{o},\mathbf{x})$~\cite{dawson2022learning}. This change is necessary to enable safe navigation in an unknown environment, where the local observations from an onboard range sensor are used to create a \ac{CBF} describing a safe set in a local, robot-centered frame. This requires some additional care when formulating the invariance condition \eqref{CBF_lie_condition}. Specifically, we assume the following observation model, where our state evolves according to \eqref{system} and the observation $\mathbf{o}$ is an unknown function (since the environment is unknown) $\zeta$ of the state $\mathbf{x}$ as $\mathbf{o} = \zeta (\mathbf{x})$.
Taking the time derivative of $h(\mathbf{o},\mathbf{x})$ results:

\small
\begin{equation}
    \frac{d}{dt}h(\mathbf{o},\mathbf{x}) = \underbrace{\nabla_\mathbf{o} h(\mathbf{o},\mathbf{x}) \cdot \nabla_\mathbf{x} \zeta(\mathbf{x}) \frac{d}{dt}\mathbf{x}}_\text{observation dynamics} + \underbrace{\nabla_\mathbf{x} h(\mathbf{o},\mathbf{x}) \cdot \frac{d}{dt}\mathbf{x}}_\text{state dynamics} \text{,}
\end{equation}
\normalsize
where we grouped the terms into the observation dynamics and the state dynamics. Computing the time derivative $\frac{d}{dt}h(\mathbf{o},\mathbf{x})$ requires knowledge of the gradient $\nabla_\mathbf{x} \zeta$.
Conversely, neither is $\zeta$ known, nor does it necessarily have a computable gradient $\nabla_\mathbf{x} \zeta$ due to possible discontinuities in the environment. 
To enable safe navigation with observation \acp{CBF} with the safety filter QP \eqref{safety_QP}, we resort to a switching set approach to avoid approximating the observation dynamics. We treat successive observations $\mathbf{o}_{i-1}$, $\mathbf{o}_{i}$ as discrete-time parameters of the function $h$, describing separate safe sets defined in the respective local frames they were obtained in. Using the fact that the union of control invariant sets is again a control invariant set, it suffices to enforce invariance of the state $\mathbf{x}$ w.r.t. either valid \ac{CBF} $h(\mathbf{o}_{i-1}, \cdot)$ or $h(\mathbf{o}_{i}, \cdot)$. Note that $\mathbf{x}$ needs to be expressed relative to the frame the corresponding observation originates from by using the relative transform $\xi(\mathbf{x})$ (Fig. \ref{fig:switching_set}). Using a constructive approach, the switching update following a new observation $\mathbf{o}_{i}$ at time $t_i$ is proposed:

\small
\begin{equation}\label{switching_update}
    \begin{aligned}
        \mathbf{o}_\text{safe} &\leftarrow \mathbf{o}_{i} \quad \text{if } h(\mathbf{o}_{i},\xi_{i}(\mathbf{x}(t_i)))) \geq 0 \\
        \xi_\text{safe} &\leftarrow \xi_{i} \quad \text{if } h(\mathbf{o}_{i},\xi_{i}(\mathbf{x}(t_i)))) \geq 0 \\
        h(\mathbf{o}&,\mathbf{x}) = h(\mathbf{o}_\text{safe},\xi_\text{safe}(\mathbf{x}))
    \text{.}
    \end{aligned}
\end{equation}
\normalsize
The above update law essentially only updates the certificate parameters ($\mathbf{o}_\text{safe}$ and $\xi_\text{safe}$), if the state is in the resulting new safe set. Otherwise, the previous parameters are applied, where $\xi_\text{safe}(\mathbf{x})$ can be determined by dead-reckoning. 
The switching set approach is visualized in Fig. \ref{fig:switching_set}. 
\begin{figure}[t]
    \centering
    \includegraphics[width=0.9\columnwidth]{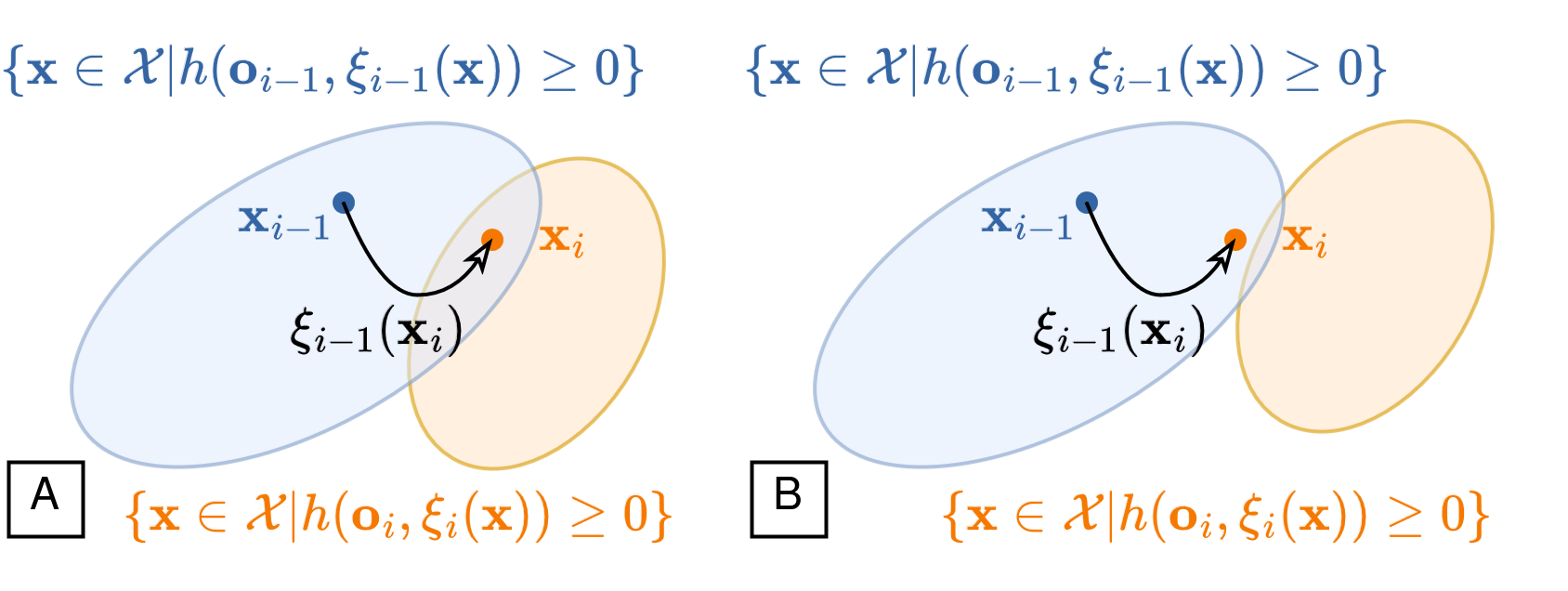}
    \vspace{-2.5ex}
    \caption{Fig. A (left): $\mathbf{x}_{i}$ is in the new safe set an the new certificate (using new latest observation) can be applied. Fig. B (right): $\mathbf{x}_{i}$ is not in the new safe set and invariance of a previous safe set is enforced.}
    \label{fig:switching_set}
    \vspace{-4ex}
\end{figure}
It is assumed that dead-reckoning occurs infrequently for short amounts of time to avoid accumulating odometry errors in the relative transform $\xi_\text{safe}$. Further, we avoid "locked-in" situations where $\{\mathbf{x} \in \mathcal{X} | h(\mathbf{o},\mathbf{x}) \geq 0 \} \cap \{\mathbf{x} \in \mathcal{X} | h(\mathbf{o}_{i},\xi_{i}(\mathbf{x}(t_i)))) \geq 0 \} = \emptyset$ by parametrization of $h$ introduced in subsection \ref{sec_SDRE}. We now show that forward invariance of this switching safe set can be achieved.
\begin{assumption}\label{assumption1}
    For an extended class $\mathcal{K_\infty}$ function $\alpha(\cdot)$, the function $h(\mathbf{o}_i,\cdot)$ is a valid \ac{CBF} for the system \eqref{system} according to Definition \ref{def_CBF} for any $\mathbf{o}_i \in \mathcal{O}$.
\end{assumption}
\begin{theorem}
    Under assumption \ref{assumption1}, the switching update \eqref{switching_update} applied at discrete time steps $t_i \in \{t_1, t_2, ..., t_e\}, e \in \mathbb{N}$ with $t_i \leq t_{i+1}$ together with any control law satisfying: 
    \small
    \begin{equation}\label{observation_CBF_lie}
    \nabla_\mathbf{x} h(\mathbf{o}_\text{safe},\xi_\text{safe}(\mathbf{x})) \cdot \frac{d}{dt}\mathbf{x} \geq -\alpha(h(\mathbf{o}_\text{safe},\xi_\text{safe}(\mathbf{x})))
    \end{equation}
    \normalsize
    $\forall t \neq t_i$ with $\mathbf{o}_\text{safe}, \xi_\text{safe}$ from \eqref{switching_update} guarantees forward invariance of the switching safe set $\mathcal{X}_\text{safe}=\{\mathbf{x} \in \mathcal{X} | h(\mathbf{o}_\text{safe},\xi_\text{safe}(\mathbf{x})) \geq 0 \}$. That is $\mathbf{x}(t_0) \in \mathcal{X}_\text{safe}$ implies $\mathbf{x}(t) \in \mathcal{X}_\text{safe} \quad \forall t\geq t_0$. 
\end{theorem}
\begin{proof}
First, we note that \eqref{observation_CBF_lie} guarantees forward invariance and asymptotic stability of the set $\mathcal{X}_\text{safe}$ for the time interval $[t_a, t_b]$ where $\forall i \quad t_i \notin [t_a, t_b]$ (\cite{ames2019control}, theorem 2). Further, the switching update \eqref{switching_update} applied at time step $t_i^-$ (before the switching update), where $\mathbf{x}(t_i^-) \in \mathcal{X}_\text{safe}$, always produces $\mathbf{o}_\text{safe},\xi_\text{safe}$ such that $\mathbf{x}(t_i^+) \in \mathcal{X}_\text{safe}$ at $t_i^+$ (after the switching update). By choosing $t_b = t_i^+$ and $\mathbf{x}(t_a) \in \mathcal{X}_\text{safe}$ we thus have $\mathbf{x}(t) \in \mathcal{X}_\text{safe} \forall t \in [t_a, t_b]$. By recursively applying this logic with time intervals $[t_0,t_1], (t_1,t_2],..., (t_{e-1},t_e], (t_e,\infty)$ with $t_0 \leq t_1$, we have  $\mathbf{x}(t) \in \mathcal{X}_\text{safe} \forall t \geq t_0$.
\end{proof}

The combined update law \eqref{switching_update}, \eqref{observation_CBF_lie} allows to consider newer and local \acp{CBF} during navigation tasks while forward invariance of the (switching) safe set is always maintained without approximation of $\nabla_\mathbf{x} \zeta$.

\subsection{Learning Observation CBFs} \label{sec:SDRE}

To learn \acp{CBF} for a general navigation task in unknown environments, the proposed method relies on sampling of the state space $\mathcal{X}$, and additional information on whether the sampled states lie within $\mathcal{X}_\text{obs}$ or $\mathcal{X}_\text{free}$. This label can be easily computed at low cost and does not require constraints to be evaluated in closed form.
Specifically, we are interested in learning the \ac{CBF} to synthesize safe control actions from arbitrary control policies without \textit{a priori} controllers and expert demonstrations, as is typically done in the literature. 
As in~\cite{dawson2022learning}, the safety controller is only learned to ensure the existence of a suitable control action during training and is discarded after training.
It thus serves as a ``plug 'n' play'' safety layer around any nominal navigation policy. Further, we assume that the origin $\mathbf{x}=0$ is part of the safe set $\mathcal{X}_\text{safe}$, i.e., that the origin can be rendered invariant by the parameterization discussed in subsection \ref{sec_SDRE}. 

Given these assumptions, the requirements posed on the learned \ac{CBF} candidate function $h(\mathbf{o},\mathbf{x})$ and safe control action $\mathbf{u}(\mathbf{o},\mathbf{x})$ are:

\small
\begin{subequations}\label{constraints}
    \begin{equation}\label{obstacle_contraint}
        h(\mathbf{o},\mathbf{x}) \leq 0 \quad \quad \quad \quad \quad \quad \quad \forall \mathbf{o} \in \mathcal{O}, \quad \forall \mathbf{x} \in \mathcal{X}_\text{obs}
    \end{equation}
    \begin{equation}\label{input_contraint}
        \mathbf{u}(\mathbf{o},\mathbf{x}) \in \mathcal{U} \quad \quad \quad \quad \quad \quad \quad \forall \mathbf{o} \in \mathcal{O}, \quad \forall \mathbf{x} \in \mathcal{X}
    \end{equation}
    \begin{equation}\label{lie_contraint}
        \frac{d}{dt} h(\mathbf{o},\mathbf{x}) \geq - \alpha(h(\mathbf{o},\mathbf{x})) \quad \forall \mathbf{o} \in \mathcal{O}, \quad \forall \mathbf{x} \in \mathcal{X} \text{.}
    \end{equation}
\end{subequations}
\normalsize

In the following, these requirements are now framed as a single multi-objective optimization problem, where we train the \ac{CBF} $h(\mathbf{o},\mathbf{x})$ and the safety controller $\mathbf{u}(\mathbf{o},\mathbf{x})$ parameterized as neural networks to satisfy these requirements.
The proposed network architecture is shown in Fig. \ref{fig:network}. In a series of networks, the observations $\mathbf{o}$ and states $\mathbf{x}$ are first passed through a latent network generating a joint latent variable $\mathbf{z}(\mathbf{o},\mathbf{x})$. This latent variable then gets passed through a controller network and a \ac{CBF} network to generate $\mathbf{u}(\mathbf{o},\mathbf{x})$ and $h(\mathbf{o},\mathbf{x})$, respectively. A static class $\mathcal{K}_{\infty}$ function $\alpha$ is then applied to $h(\mathbf{o},\mathbf{x})$ to evaluate condition \eqref{CBF_lie_condition}. During training, the network outputs are passed to a function enforcing requirements \eqref{constraints} and the network weights are updated using stochastic gradient descent. By using the latent representation $\mathbf{z}(\mathbf{o},\mathbf{x})$, the input to the \ac{CBF} network (e.g. $\mathbf{z}(\mathbf{o},\mathbf{x})$) is sensitive to penalties placed violation of \eqref{input_contraint} (in contrast to simply using the controller network and \ac{CBF} network directly on $(\mathbf{o},\mathbf{x})$), which improves the learning rate. The reason for this is that the input to the controller network and the \ac{CBF} network is updated by all loss functions.

\begin{figure*}[t]
    \centering
    \includegraphics[width=0.85\textwidth]{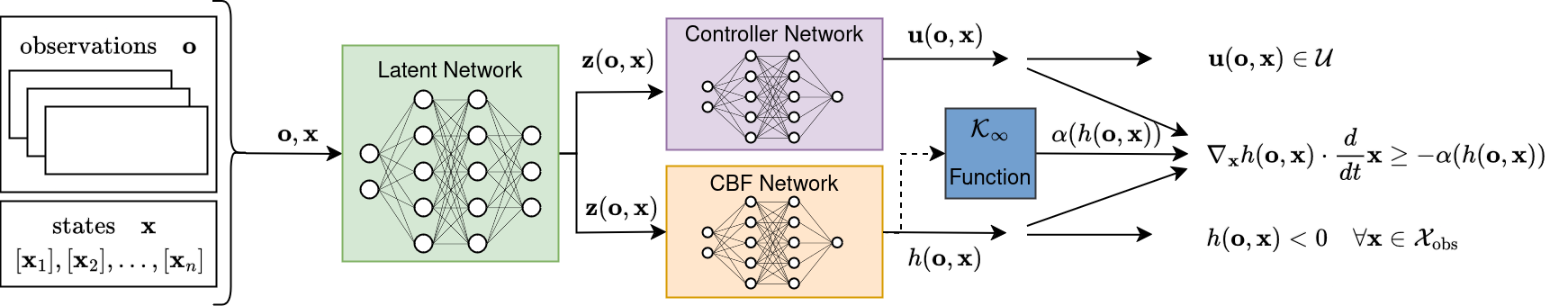}
    \vspace{-1ex}
    \caption{Proposed network architecture for jointly training a safety controller and \ac{CBF}. The observations $\mathbf{o}$ and states $\mathbf{x}$ are first passed through a latent network to generate a latent representation $\mathbf{z}(\mathbf{o},\mathbf{x})$, which is then passed through a controller network and a \ac{CBF} network to generate the control input $\mathbf{u}(\mathbf{o},\mathbf{x})$ and \ac{CBF} value $h(\mathbf{o},\mathbf{x})$, respectively. The \ac{CBF} value $h(\mathbf{o},\mathbf{x})$ is then passed through a (static) class $\mathcal{K}_{\infty}$ function $\alpha$. The method relies only on instantaneous observations $\mathbf{o}$ and not a map.}
    \label{fig:network}
    \vspace{-3ex}
\end{figure*}

\subsection{Joint Learning of Controller and \ac{CBF}} \label{sec_SDRE}

We now propose a formulation to jointly learn safe controllers and the corresponding \acp{CBF} using an adaptation of matrix-valued \ac{SDRE}. For the sake of legibility, we omit any dependence of $h$ on $\mathbf{o}$ in the following derivation, e.g. we replace $h(\mathbf{o},\mathbf{x})$ with $h(\mathbf{x})$.
Our starting point is the \ac{HJB} equation for a value function $V(\mathbf{x}) \geq 0$ and stage cost $l(\mathbf{x}, \mathbf{u})$:

\small
\begin{equation}\label{HJB}
    \frac{d}{dt} V(\mathbf{x}) + \underset{\mathbf{u} \in \mathcal{U}}{\text{sup}} [l(\mathbf{x}, \mathbf{u})] = 0.
\end{equation}
\normalsize
The \ac{SDRE} can now be derived as in \cite{cimen_SDRE}, assuming a structure of $V(\mathbf{x})$ as $V(\mathbf{x}) = \mathbf{x}^T P(\mathbf{x}) \mathbf{x}$ together with a stage cost $l(\mathbf{x}, \mathbf{u}) = \mathbf{x}^T Q(\mathbf{x}) \mathbf{x} + \mathbf{u}^T R(\mathbf{x}) \mathbf{u}$, for symmetric matrices $P(\mathbf{x})$, $Q(\mathbf{x})$, $R(\mathbf{x})$ as shown in  \cite{cimen_SDRE}. Inserting these into the \ac{HJB} equation yields:

\small
\begin{equation}\label{HJB_expanded}
    \nabla_\mathbf{x} V(\mathbf{x})^T \cdot  \frac{d}{dt} \mathbf{x} + \underset{\mathbf{u} \in \mathcal{U}}{\text{sup}} [\mathbf{x}^T Q(\mathbf{x}) \mathbf{x} + \mathbf{u}^T R(\mathbf{x}) \mathbf{u}] = 0.
\end{equation}
\normalsize
By inserting the gradient term:

\small
\begin{equation}\label{gradient_term}
    \nabla_\mathbf{x} V(\mathbf{x}) = \underbrace{[2 P(\mathbf{x}) + \mathbf{x}^T (\nabla_\mathbf{x} P(\mathbf{x}))]}_{2 S(\mathbf{x})} \mathbf{x}
\end{equation}
\normalsize
and solving for $\mathbf{u}$ by setting the derivative of \eqref{HJB_expanded} w.r.t. $\mathbf{u}$ to zero, we arrive at the optimal state feedback law:

\small
\begin{subequations}\label{SDRE_augmented}
    \begin{equation}\label{SDRE_augmented_law}
        \mathbf{u}^* = - R^{-1} B^T S \mathbf{x}
    \end{equation}
    \begin{equation}\label{SDRE_augmented_condition}
        \mathbf{x}^T[A^T S + S^T A - S^T B R^{-1} B^T S + Q]\mathbf{x} = 0.
    \end{equation}
\end{subequations}
\normalsize
This equation is inspired by the \ac{SDRE} \cite{cimen_SDRE}, which follows from inserting approximation $\nabla_\mathbf{x} V(\mathbf{x}) = 2 P(\mathbf{x}) \mathbf{x}$ in \eqref{HJB_expanded}. The \ac{SDRE} is then given by replacing the non-symmetric matrix $S(\mathbf{x})$ in \eqref{SDRE_augmented} by the symmetric matrix $P(\mathbf{x})$, allowing standard Riccati solvers can be used to compute $P(\mathbf{x})$. This approximation leads to the lack of global stability properties for \ac{SDRE} control laws as the additional term $\mathbf{x}^T (\nabla_\mathbf{x} P(\mathbf{x})) \mathbf{x}$ affects the behavior of $\frac{d}{dt} V(\mathbf{x})$ far from the origin $\mathbf{x}=0$ \cite{cimen_SDRE}. At this point, we emphasize that the global satisfaction of \eqref{SDRE_augmented_condition} implies global stability of the equilibrium $\mathbf{x}=0$ when applying \eqref{SDRE_augmented_law} (since it is a reformulation of the \ac{HJB}).

We now aim to formulate a scheme that allows us to jointly generate \acp{CBF} and corresponding safe controllers from value functions given in the form of \eqref{SDRE_augmented}. Therefore, we must replace the stage cost terms $\mathbf{x}^T Q(\mathbf{x}) \mathbf{x} + \mathbf{u}^T R(\mathbf{x}) \mathbf{u}$ in the optimal control formulation by their \ac{CBF} counterparts, e.g. the extended class $\mathcal{K}_\infty$ function $\alpha$.
We now define a candidate \ac{CBF} using the value function $V(\mathbf{x})$ as:

\small
\begin{equation}\label{CBF_parametrisation}
    h(\mathbf{x}) = 1 - V(\mathbf{x})
                  = 1 - \mathbf{x}^T P(\mathbf{x}) \mathbf{x}.
\end{equation}
\normalsize
For any $P(\mathbf{x})$ which satisfies \eqref{SDRE_augmented} globally on $\mathcal{X}$, $h(\mathbf{x})$ indeed has invariant superlevel sets (since $V(\mathbf{x})$ has invariant sublevel sets) for the controlled system with feedback law as in \eqref{SDRE_augmented_law}. However, any solution to \eqref{SDRE_augmented_condition} does imply optimality for some $Q(\mathbf{x}), R(\mathbf{x})$, yet the choice of these cost matrices \textit{a priori} to achieve a large safe set remains unclear. Instead, the \ac{CBF} must satisfy \eqref{CBF_lie_condition} by definition, which translates to:

\small
\begin{equation}\label{CBF_SDRE}
    \frac{d}{dt} h(\mathbf{x}) = -\mathbf{x}^T[A^T S + S^T A - 2S^T B R^{-1} B^T S]\mathbf{x} \geq -\alpha(h(\mathbf{x})).
\end{equation}
\normalsize
To bring the above equation into a matrix-valued form, we expand the right hand term to a quadratic form as:

\small
\begin{equation}\label{kappa_expansion}
    \alpha(h(\mathbf{x})) = \alpha(1-\mathbf{x}^T P(\mathbf{x}) \mathbf{x}) = \mathbf{x}^T [\frac{\alpha(h(\mathbf{x}))}{h(\mathbf{x})}(\frac{1}{||\mathbf{x}||^2} \mathbb{I} - P)] \mathbf{x}
\end{equation}
\normalsize
and reinsert into~\eqref{CBF_SDRE} to obtain the necessary condition: 

\small
\begin{equation}\label{CBF_SDRE_scalar}
    \mathbf{x}^T[A^T S + S^T A - 2S^T B R^{-1} B^T S - \frac{\alpha(h(\mathbf{x}))}{h(\mathbf{x})}(\frac{1}{||\mathbf{x}||^2} \mathbb{I} - P)]\mathbf{x} \leq 0
    \text{.}
\end{equation}
\normalsize
This equation can be interpreted as a lower bound of~\eqref{SDRE_augmented}, where the terms of the stage cost $S^T B R^{-1} B^T S + Q$ have been replaced with the expansion of $\alpha$ in \eqref{kappa_expansion}. Here, the quadratic form of $\alpha$ acts as a lower bound on the stage cost. The lower bound is explicitly given by:

\small
\begin{equation}\label{stage_cost_bound}
    \mathbf{x}^T[\frac{\alpha(h(\mathbf{x}))}{h(\mathbf{x})}(\frac{1}{||\mathbf{x}||^2} \mathbb{I} - P)]\mathbf{x} \leq \underbrace{\mathbf{x}^T Q \mathbf{x} + \mathbf{u}^{*T} R \mathbf{u}^{*}}_{l(\mathbf{x}, \mathbf{u}^*)}
    \text{.}
\end{equation}
\normalsize
Note that, unlike the \ac{SDRE} framework, the bound does allow for $Q(\mathbf{x})\prec0$ since the \ac{CBF} framework only ensures control invariance of all sublevel sets where $h(\mathbf{x}) \leq 0$.
%

So far, all equations have been expressed in a scalar form, where we only require ~\eqref{CBF_SDRE_scalar} to hold in scalar form due to the explicit dependence of the matrix $P$ on $\mathbf{x}$. Strictly speaking, only the scalar-valued Eq,~\eqref{CBF_SDRE_scalar} is required to hold globally. While being a sufficient condition, the matrix-valued equation:

\small
\begin{equation}\label{CBF_SDRE_matrix}
    A^T S + S^T A - 2S^T B R^{-1} B^T S - \frac{\alpha(h(\mathbf{x}))}{h(\mathbf{x})}(\frac{1}{||\mathbf{x}||^2} \mathbb{I} - P) \leq 0
\end{equation}
\normalsize
is a stronger requirement than that the invariance condition \eqref{CBF_lie_condition} poses. 
However, we observe that these stronger conditions show favorable properties in terms of training stability and generalizability when jointly learning a control policy and a corresponding \ac{CBF}. The reason for this is that the scalar condition \eqref{CBF_lie_condition} itself is not informative enough (as a loss function) to jointly learn safe control policies and \acp{CBF} for higher order systems by pure sampling. For higher-order systems, \eqref{CBF_lie_condition} only provides information along $\nabla_\mathbf{x} h(\mathbf{x})$, which leads to local minima in training. Therefore, previous works use first-order systems \cite{dawson2022learning} or resort to using a stable reference controller to guide the training process away from local minima \cite{dawson2022safe}. We overcome this issue by imposing the stricter condition \eqref{CBF_SDRE_matrix}. 
This is generally more restrictive than the scalar form \eqref{CBF_SDRE_scalar} due to the gradient coupling terms in $S$ but provides richer information for learning the matrix $P$ from data. It is important to emphasize that~\eqref{CBF_SDRE_matrix} is always satisfied for some $\alpha$ in the trivial case where $P$ and $R$ are independent of $\mathbf{x}$, assuming,~\eqref{SDRE_augmented} is satisfied for some $Q$ and $R$.
We will now make use of \eqref{CBF_SDRE_matrix} to learn $P(\mathbf{x})$ and $R(\mathbf{x})$ from sampled data.

\subsection{Training the Networks}
To create a single multi-objective cost function to train the networks depicted in Figure~\ref{fig:network}, we follow an approach similar to \cite{dawson2022safe}, taking the weighted average of the cost functions corresponding to the requirements \eqref{constraints}. However, we deviate from \cite{dawson2022safe} by not using any prior safe states and stable controller but insert a condition enforcing satisfaction of \eqref{CBF_SDRE_matrix} to jointly learn \ac{CBF} and safe controller.
The total cost of a minibatch of samples $(\mathbf{o},\mathbf{x})^j$ of size $b$ is given by:

\small
\begin{subequations}\label{cost_fun}
\begin{align}
    J = \lambda_1 &\sum_{(\mathbf{o},\mathbf{x})^j\in (\mathcal{O},\mathcal{X})_\text{obs}} \text{ReLU}(h(\mathbf{o}^j,\mathbf{x}^j)+\epsilon_1) \quad + \label{cost_fun_1}\\
        \lambda_2 &\sum_{(\mathbf{o},\mathbf{x})^j\in (\mathcal{O},\mathcal{X})_\text{safe}} (1-h(\mathbf{o}^j,\mathbf{x}^j)) \quad + \label{cost_fun_2}\\
        \lambda_3 &\sum_{(\mathbf{o},\mathbf{x})^j\in (\mathcal{O},\mathcal{X})}  \| \mathbf{u}(\mathbf{o}^j,\mathbf{x}^j) - \text{proj}_{\mathcal{U}}(\mathbf{u}(\mathbf{o}^j,\mathbf{x}^j)) \| \quad + \label{cost_fun_3}\\
        \lambda_4 &\sum_{(\mathbf{o},\mathbf{x})^j\in (\mathcal{O},\mathcal{X})} \sum_{i=1}^n \text{ReLU} ((\text{eig}(K(\mathbf{o}^j,\mathbf{x}^j))^i + \epsilon_2)), \label{cost_fun_4}
\end{align}\\
\end{subequations}
\normalsize
where $\lambda_{1-4}\geq0$ are weighting parameters and $\epsilon_{1,2}\geq0$ are slack parameters to enforce stricter satisfaction and generalization of the requirements in \eqref{constraints}. The first term~\eqref{cost_fun_1} penalizes the intersection of the safe set $(\mathcal{O},\mathcal{X})_\text{safe}$ with the obstacle set $(\mathcal{O},\mathcal{X})_\text{obs}$, while~\eqref{cost_fun_2} infers expansion of the safe set, and~\eqref{cost_fun_3} penalizes violations of requirement 2.
In~\eqref{cost_fun_4}, $K$ corresponds to the symmetric left-hand side of \eqref{CBF_SDRE_matrix} and penalizes violations of \eqref{CBF_SDRE_matrix}, enforcing the requirement of control invariance. Further, we add regularization to the term $\frac{d}{dt}P(\mathbf{x})$. The cost $J$ in \eqref{cost_fun} is then applied to a batch of training samples $(\mathbf{o}, \mathbf{x})^j$ where $\mathbf{o}^j$ is drawn from a distribution of training observations and $\mathbf{x}^j$ is randomly sampled from $\mathcal{X}$.

To improve the satisfaction of requirement \eqref{obstacle_contraint}, an additional loss:

\small
\begin{equation}\label{obstacle_loss}
    J_\text{obs} = \lambda_5 \sum_{(\mathbf{o},\mathbf{x})^j\in \delta(\mathcal{O},\mathcal{X})_\text{obs}} \text{ReLU}(h(\mathbf{o}^j,\mathbf{x}^j)+\epsilon_1)
\end{equation}
\normalsize
is computed over a set of ``boundary samples'' $(\mathbf{o},\mathbf{x})^j \in \delta (\mathcal{O},\mathcal{X})_\text{obs}$ is added to \eqref{cost_fun}, where $\delta (\mathcal{O},\mathcal{X})_\text{obs}$ denotes the boundary of $(\mathcal{O},\mathcal{X})_\text{obs}$ (\eg, states on the obstacle boundary).

\section{Implementation for Quadrotor Navigation}\label{sec:quadrotor}
We apply the proposed methodology for the case of the planar dynamics of a quadrotor aerial robot with state variable $\mathbf{x}$ receiving observations $\mathbf{o}$ of its environment from an onboard LiDAR. The observation $\mathbf{o}$ is considered to be a vector of $N$ range measurements in the horizontal plane. Neglecting the attitude dynamics, the state of the system is the linear position and velocity  $\mathbf{x} = [x, y, v_x, v_y]^T$ and the input is the linear acceleration $\mathbf{u} = [a_x, a_y]^T$, yielding the linear system dynamics:

\scriptsize
\begin{equation}\label{quadrotor}
\frac{d}{dt}
    \begin{bmatrix}
        x \\
        y \\
        v_x \\
        v_y 
    \end{bmatrix} = 
    \underbrace{
    \begin{bmatrix}
        0 & 0 & 1 & 0 \\
        0 & 0 & 0 & 1 \\
        0 & 0 & 0 & 0 \\
        0 & 0 & 0 & 0 
    \end{bmatrix}
    }_{A(\mathbf{x})}
    \begin{bmatrix}
        x \\
        y \\
        v_x \\
        v_y 
    \end{bmatrix} +
    \underbrace{
    \begin{bmatrix}
        0 & 0 \\
        0 & 0 \\
        1 & 0 \\
        0 & 1 
    \end{bmatrix}
    }_{B(\mathbf{x})}
    \begin{bmatrix}
        a_x \\
        a_y 
    \end{bmatrix} \text{.}
\end{equation}
\normalsize
Further, we assume an input constraint $||\mathbf{u}|| \leq \SI{2}{\metre\per\second^2}$. The observation $\mathbf{o}$ is considered to be a vector of $N=32$ equally-spaced range measurements in the horizontal plane. 
\begin{figure}
    \centering
    \includegraphics[width=0.95\columnwidth]
    {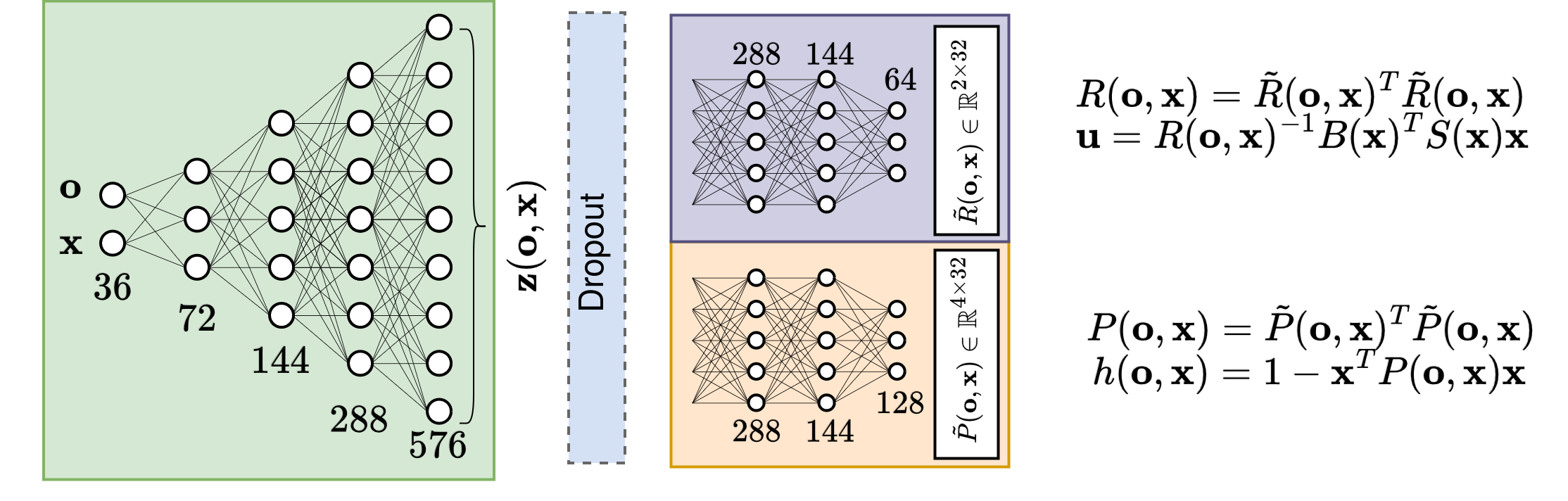}
    \vspace{-1ex}
    \caption{Network structure and dimensions for quadrotor navigation on $xy$. The output of the \ac{CBF} and controller network consist of the sum of 32 ``feature matrices'' each to allow superposition of features in the final output.}
    \label{fig:network2}
    \vspace{-3ex}
\end{figure}

\subsection{Dataset}
To train the neural \ac{CBF} to predict safe sets for discontinuous environments with varying obstacle density, a dataset of $10^4$ simulated training observations are collected from a randomized environment with varying obstacle placement and density in the Aerial Gym Simulator~\cite{kulkarni2023aerial}. To sample the state space $\mathcal{X}$, positions around a robot-aligned frame are sampled uniformly in bearing and distance, while velocities in both axes are sampled uniformly from the interval $[-v_\text{max}, v_\text{max}]$. Information whether a sample $(\mathbf{o}, \mathbf{x})^j$ belongs to $\mathcal{X}_\text{obs}$ is computed by projecting the position of $\mathbf{x}^j$ and checking collisions with $\mathbf{o}^j$. Furthermore, it is assumed $\mathbf{x} \in \mathcal{X}_\text{obs}$ if the position falls outside a $4~\textrm{m}$ radius. To create samples on the obstacle boundaries for evaluating \eqref{obstacle_loss}, raycasting with collision checking in $\mathbf{o}^j$ is used. Note that during training, for each observation $\mathbf{o}^j$ a new state $\mathbf{x}^j$ is sampled.

\subsection{Training}
For all the results presented in this work (simulation and experiments), a single set of weights for the network shown in Fig. \ref{fig:network2} is used. We use ELU activations for all hidden layers and apply a dropout with $p=0.1$ to the latent variable. The constraints $P(\mathbf{o}, \mathbf{x}) \succ 0$ and $R(\mathbf{o}, \mathbf{x}) \succ 0$ are always satisfied as $P(\mathbf{o}, \mathbf{x})$ and $R(\mathbf{o}, \mathbf{x})$ are computed as the inner product of feature matrices $\tilde{P}(\mathbf{o}, \mathbf{x})$ and $\tilde{R}(\mathbf{o}, \mathbf{x})$ with itself, respectively. The gradients $\nabla_\mathbf{x} h (\mathbf{o}, \mathbf{x})$ and $\nabla_\mathbf{x} P (\mathbf{o}, \mathbf{x})$ needed for \eqref{observation_CBF_lie} and \eqref{gradient_term} are easily computed via the autograd function in PyTorch\footnote{\url{https://pytorch.org/}} during training and at runtime. During a training cycle, a batch $\{(\mathbf{o}, \mathbf{x})^j\}_{j=1}^b$ of size $b$ is passed through all networks to compute the loss function \eqref{cost_fun}. The weights of all networks are then updated synchronously using the Adam optimizer. We train for 100 epochs with batch size $32\times128$ (32 observations, 128 states each) without additional boundary samples and then proceed to train for 250 epochs with batch size $32\times256$ with additional boundary samples, where the loss \eqref{obstacle_loss} is added to \eqref{cost_fun}. We set $\alpha(h)$ to

\small
\begin{equation}
\alpha(h) = 
\begin{cases}
    2\cdot h,& \text{if } h \geq 0\\
    \frac{1}{0.5 + |h|}             & \text{otherwise}
\end{cases}
,
\end{equation}
\normalsize
where $\alpha$ is selected based on desired behavior of the filter and must admit a solution to \eqref{CBF_SDRE_matrix} given the input constraints.
The evaluation statistics shown in Table \ref{training_eval} show the resulting high degree of satisfaction of requirements \eqref{constraints} on a holdout dataset of $10^3$ observations with $256$ random states each.

\begin{table}
    \centering
    \caption{Evaluation Statistics}
    \label{training_eval}
    \begin{tabular}{ |p{6.25cm}||p{1.25cm}|  }
     \hline
     Constraint from \eqref{constraints} &  violation rate [\%]\\
     \hline
     $h(\mathbf{o},\mathbf{x}) \leq 0 \quad  \forall (\mathbf{o},\mathbf{x}) \in (\mathcal{O}, \mathcal{X})_\text{obs}$&  0.084\\
     $\mathbf{u}(\mathbf{o},\mathbf{x}) \in \mathcal{U}  \quad \forall (\mathbf{o},\mathbf{x}) \in (\mathcal{O}, \mathcal{X})$ &  0.321\\
     $\frac{d}{dt} h(\mathbf{o},\mathbf{x}) \geq - \alpha(h(\mathbf{o},\mathbf{x}))  \forall (\mathbf{o},\mathbf{x}) \in (\mathcal{O}, \mathcal{X})$ & 1.133\\
     \hline
    \end{tabular}
    \vspace{-4ex}
\end{table}

\subsection{Recursive Filtering}
To apply the trained \acp{CBF} in a safety filter, for every new observation, we first apply the observation update \eqref{switching_update} and then employ condition \eqref{observation_CBF_lie} in a QP-based safety filter \eqref{safety_QP}.
In the case that $\mathbf{o}_\text{safe}$ and $\xi_\text{safe}$ are not updated in \eqref{switching_update}, an estimate of $\xi_\text{safe}$ is obtained by integrating the velocity since the last update.
This ``dead-reckoning'' approach is justifiable over short time horizons and eliminates the need for a full pose estimate inside the safety filter. To avoid dead-reckoning for too long, we limit its horizon to $\eta$ and penalize it in the optimization problem \eqref{constraint_satisfaction_slack_vars}. Accounting for the robot dimensions, we evaluate $h(\mathbf{o},\mathbf{x})$ for four states $\mathbf{x}_j$ on a bounding box around the robot and solve the (soft constrained) optimization problem with the slack variable $\delta$ and the slack parameter $\lambda_s \gg 0$:

\small
\begin{equation}\label{constraint_satisfaction_slack_vars}
    \begin{aligned}
    \mathbf{u}^*  &= \quad \underset{\mathbf{u} \in \mathcal{U}}{\text{argmin} } \quad  ||\mathbf{u} - \mathbf{u}_\text{ref}||^2  + \lambda_s\delta - \lambda_g \mathbf{g} \cdot \mathbf{u}\\
         \textrm{s.t.} & \quad  \nabla_{\mathbf{x}_j} h(\mathbf{o},\mathbf{x}_j) \cdot [f(\mathbf{x}_j) + B(\mathbf{x}_j)\mathbf{u}]  + \\
        & \quad \alpha(h(\mathbf{o},\mathbf{x}_j)) \geq - \delta \quad \forall j \in \{1,2,3,4\} \\
        & \quad \delta \geq 0\\
    \end{aligned}
\end{equation}
\normalsize
to enforce constraint satisfaction for the points $\mathbf{x}_j$. The term $\lambda_g$ is a tuning parameter, and $\mathbf{g}=\sum_j h(\mathbf{o}_{i},\mathbf{x}_j)$ during dead reckoning, and $\mathbf{g}=0$ otherwise.

\section{Evaluation Studies}\label{sec:xp}

\subsubsection{Aerial Robot}

\begin{figure*}[hbt!]
    \centering
    \includegraphics[width=0.95\textwidth]{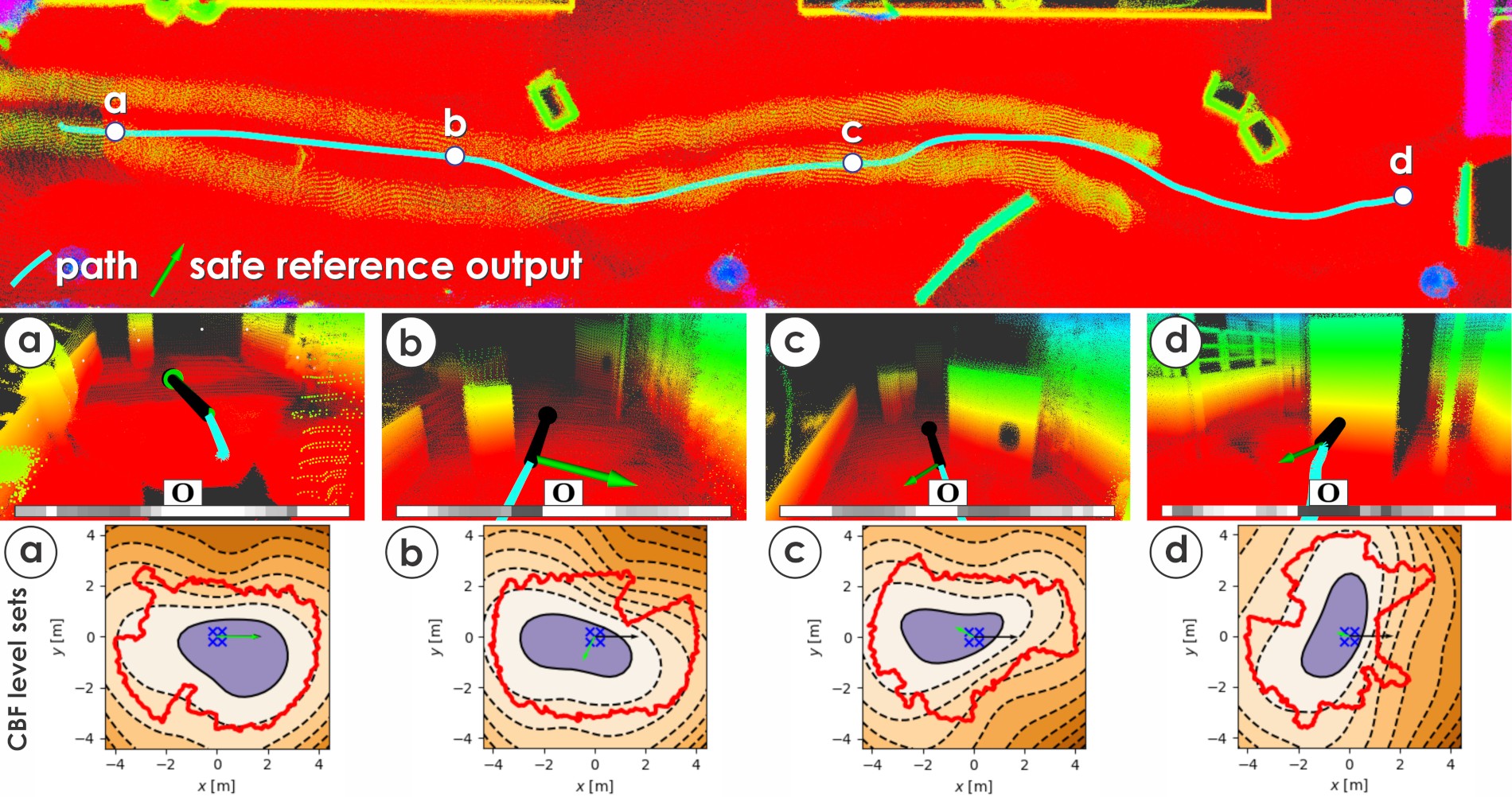}  
    \vspace{-1.5ex}
    \caption{Experimental evaluation of the proposed safety filter in a corridor. The safety filter receives a constant acceleration input of $\SI{2}{\metre\per\second^2}$ in the vehicle frame (black arrow) and produces a safe reference output (green arrow). The safety filter passes through safe reference inputs (a), deflects the multirotor when passing obstacles (b\&c) and brings the robot to a full stop in front of the final panel (d). The observation (range image) at each instance is shown. \ac{CBF} level sets are drawn for all instances (a-d) at the current velocity in the vehicle frame where purple denotes $h\geq0$ and red denotes the obstacle boundary. The position of points $\mathbf{x}_j$ is shown as blue crosses.}
    \label{fig:experiment_1}
    \vspace{-3ex}
\end{figure*}

The conducted experiments relied on a custom-made quadrotor with dimensions $0.52\times 0.52\times0.31~\textrm{m}$ and a mass of $2.58~\textrm{kg}$. The system integrates PX4-based autopilot avionics for low-level control, alongside an NVIDIA Orin NX compute board, an Ouster OS0-64 LiDAR and a VectorNav VN-100 IMU. We use LiDAR localization as in~\cite{khattak2020complementary} to estimate the odometry of the robot. The input LiDAR data (20~\textrm{Hz}) of the horizontal plane is limited to a range of $4~\textrm{m}$ and binned angle-wise to create an observation vector $\mathbf{o}$ as in training ($N=32$). In the following, $\eta$ is set to $3$. The inference time of the neural network including gradient computation is $4.4$ ms. Critically, it is underlined again that the proposed approach does not utilize a reconstructed map but only the instantaneous LiDAR scan. The robot's nominal policy for the purposes of testing are commanded horizontal accelerations in the vehicle frame while yaw is regulated against the inertial frame. The proposed safety filter alters the commanded acceleration to ensure safety.

\subsection{Simulation Study}

\begin{table}
    \caption{Success rate in collision avoidance}
    \label{simulation_table}    
    \centering
    \begin{tabular}{ |p{1.5cm}||p{1.5cm}||p{1.0cm}||p{1.0cm}||p{1.0cm}|  }
     \hline
      Noise level & time delay &  $p=3$ & $p=5$ & $p=10$\\
     \hline
     $\sigma=0.0~\textrm{m}$ & $\tau=0.0~\textrm{s}$ & 97.0 & 97.2 & 99.7\\
     \hline
     $\sigma=0.02~\textrm{m}$ &$\tau=0.0~\textrm{s}$ & 97.8 & 99.2 & 99.7\\
     \hline
     $\sigma=0.02~\textrm{m}$ &$\tau=0.1~\textrm{s}$ & 95.0 & 97.0 & 98.3\\
     \hline
    \end{tabular}
    \vspace{-4ex}
\end{table}

We evaluate the proposed safety filter in a randomized simulation study in the Aerial Gym Simulator \cite{kulkarni2023aerial}. An enclosed environment of $20\times10~\textrm{m}$ is constructed, bounded by 4 walls and a number of $p$ vertical pillars of radii $0.75~\textrm{m}$ to $1.0~\textrm{m}$ are placed randomly. We simulate the quadrotor described previously with the simplified dynamics model \eqref{quadrotor} for $10^3$ rollouts. During each rollout, the robot is spawned at one end of the environment and receives a constant, unsafe reference input $\mathbf{u}_\text{ref}=[ \SI{2}{\metre\per\second^2}, \SI{0}{\metre\per\second^2}]^T$, driving it towards the pillars or the walls. The epoch terminates after $10$~\textrm{s} or on collision. Without the intervention of the safety filter, the robot is guaranteed to collide during every rollout. We add artificial Gaussian noise with std $\sigma$ to the LiDAR image before binning and a time delay $\tau$ to the system \eqref{quadrotor} to account for the attitude dynamics.
The ablated success rates (\eg, no collisions) over obstacle densities are shown in Table \ref{simulation_table}. We note that noise on the LiDAR image improves the success rate, as we take the minimum of each bin.

\subsection{Experimental Studies}
\begin{figure}
    \centering
    \includegraphics[width=0.95\columnwidth]{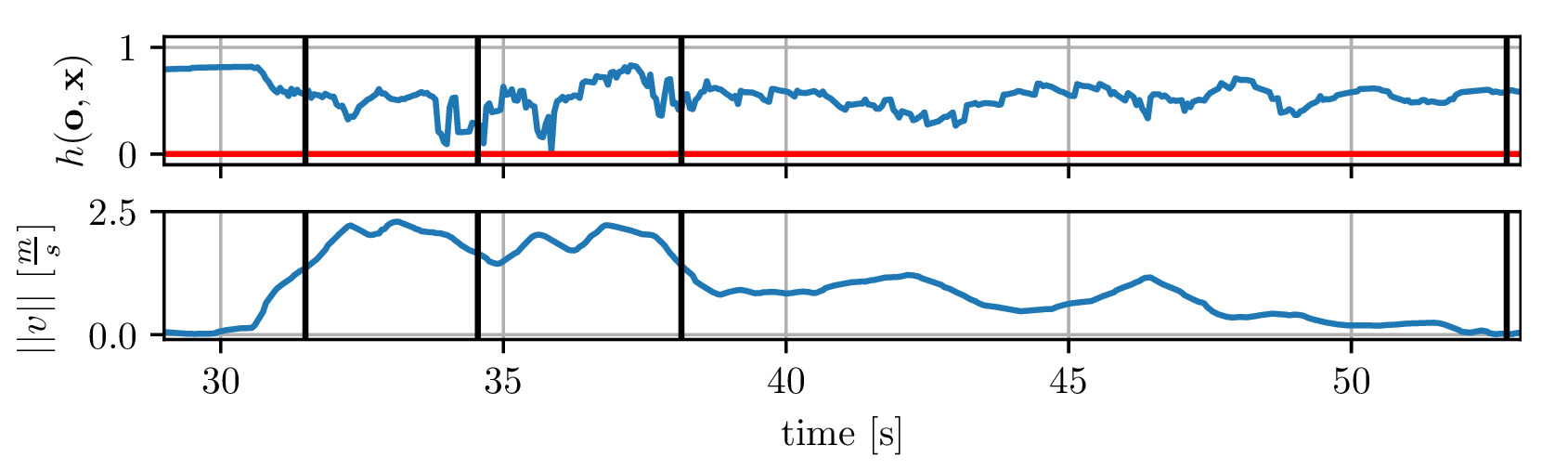}
    \vspace{-2ex}
    \caption{Value of $h(\mathbf{o},\mathbf{x})$ during the corridor experiment (top) and ${\scriptscriptstyle ||v||=\sqrt{v_x^2+v_y^2}}$ (bottom). The times of instances (a-d) are marked as black lines. $h(\mathbf{o},\mathbf{x})$ is not equal to $1$ at rest since it is evaluated at $\mathbf{x}_j$.}
    \label{fig:cbf_val1}
    \vspace{-4ex}
\end{figure}
For the experimental study, two hardware experiments were carried out in indoor and outdoor environments. In both cases, the safety filter operates at a rate of $20~\textrm{Hz}$. In the first experiment, the robot's nominal policy is a constant commanded body-frame acceleration of $\mathbf{u}_\text{ref}=[ \SI{2}{\metre\per\second^2}, \SI{0}{\metre\per\second^2}]^T$. In an obstacle-filled hallway, the proposed safety filter modifies $\mathbf{u}_\text{ref}$ to deflect the robot away from obstacles and brings the robot to a full stop in front of the final obstacle, avoiding collisions despite the unsafe reference input along the way. A summary of the experiment is shown in Fig. \ref{fig:experiment_1}. Considering the value of $h(\mathbf{o},\mathbf{x})$ in Fig. \ref{fig:cbf_val1} , we see that the proposed safety filter manages to keep the robot within $\mathcal{X}_\text{safe}$ over a wide range of velocities during the entire experiment. 

The second experiment is conducted in a forest environment. Here, $\mathbf{u}_\text{ref}$ is given by the manual input of a human operator, intentionally trying to guide the robot into obstacles like tree trunks and foliage. A short sequence of the  experiment is briefly shown in Fig. \ref{fig:experiment_2}. The proposed safety filter successfully avoids collision by decelerating the robot and deflecting it away from obstacles.

\begin{figure}
    \centering
    \includegraphics[width=0.99\columnwidth]{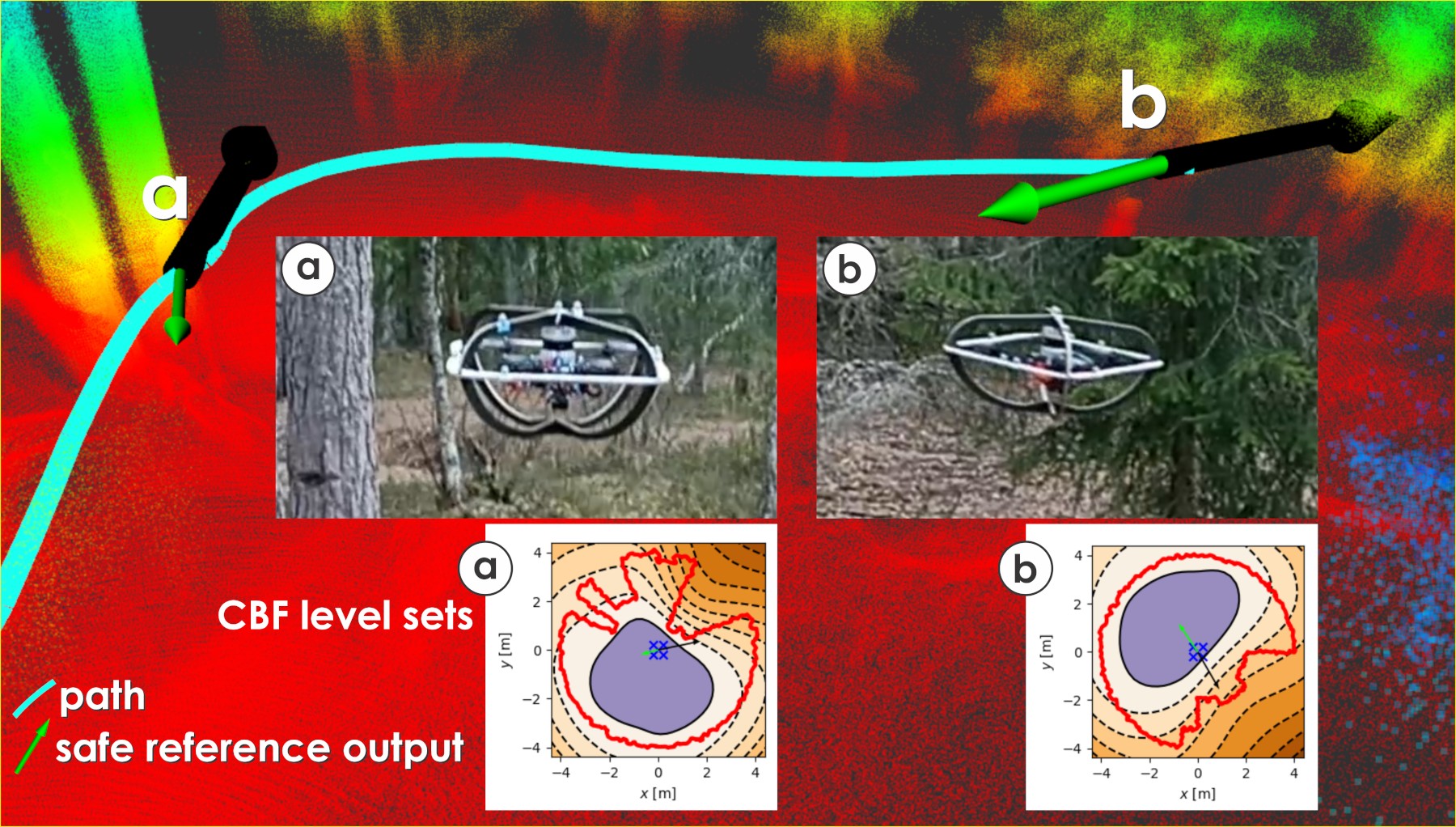} 
    \vspace{-2ex}
    \caption{Experimental evaluation in a forest environment with an adversarial human operator as a reference controller. At instance (a), the human operator tries to crash the multirotor into a tree trunk, while at (b), the obstacle is a spruce. The proposed safety filter counteracts the humans actions, avoiding collision. \ac{CBF} level sets are shown in purple for the current velocity.}
    \label{fig:experiment_2}
    \vspace{-4ex}
\end{figure}

\section{Conclusions}\label{sec:concl}

This work presented a new approach to using learned \acp{CBF} for map-less safe navigation of aerial robots in unknown environments.
The \ac{CBF} is learned without any expert controller, using an adaptation of the \ac{SDRE}. Deep neural networks are trained to represent the \ac{CBF} using only range observations, and the state of the robot.
The learned \ac{CBF} is utilized in a reactive safety filter that enforces collision avoidance. Training is performed relying purely on simulated data. Randomized simulation studies and real-world experiments demonstrate the applicability of the proposed method in ensuring the safety of the platform.

\bibliographystyle{IEEEtran}
\bibliography{BIB/main}
\end{document}